\documentclass[a4paper,10pt]{article}

\usepackage{amsthm,amsfonts,amsmath,amssymb,amscd} 
\usepackage{mathtools} 

\usepackage{color}
\usepackage{multirow}
\usepackage{array}
\usepackage{tabu}
\usepackage{float} 
\usepackage{cleveref}
\usepackage{amsmath}
\usepackage{bbm}
\usepackage{chicago}

\usepackage{cmap}			
\usepackage[T1,T2A]{fontenc}		
\usepackage[utf8]{inputenc}		
\usepackage[english]{babel}	
\usepackage{breakcites}
\usepackage{breakurl}

\usepackage{graphicx}
\usepackage{wrapfig}

\newcommand{\Indicator}[1]{\mathbbm{1}_{#1}}

\newcommand{\e}{\varepsilon}
\newcommand{\A}{\mathcal{A}}
\newcommand{\F}{\mathcal{F}}
\renewcommand{\H}{\mathcal{H}}
\newcommand{\M}{\mathcal{M}}
\newcommand{\Y}{\mathcal{Y}}
\newcommand{\X}{\mathcal{X}}
\newcommand{\bR}{\mathbb{R}}
\newcommand{\fR}{\mathfrak{R}}
\newcommand{\fK}{\mathfrak{K}}
\newcommand{\bP}{\mathbb{P}}
\newcommand{\bE}{\mathbb{E}}
\newcommand{\bH}{\mathbb{H}}

\newtheorem{theorem}{Theorem}
\newtheorem{lemma}{Lemma}

\title{
  Tight Risk Bounds for Multi-Class Margin Classifiers
}

\author{
Yury Maximov \\ 
Predictive Modeling and Optimization Department\\
Institute of Information Transmission Problems \\
Moscow, Bolshoy Karenty 19/1, 127051 \\
\texttt{yurymaximov@iitp.ru} \\
\and
Daria Reshetova \\
Predictive Modeling and Optimization Laboratory\\
Moscow Institute of Physics and Technology \\
Moscow, Kerchenskaya 1a/1, 117303 \\[1ex]
Predictive Modeling and Optimization Department\\
Institute of Information Transmission Problems \\
Moscow, Bolshoy Karenty 19/1, 127051 \\
\texttt{reshetova@phystech.edu} \\
}
\date{}

\begin{document}

\maketitle

\begin{abstract}
We consider a problem of risk estimation for large-margin multi-class classifiers.~We propose a novel risk bound for the multi-class classification problem. The bound involves the marginal distribution of the classifier and the Rademacher complexity of the hypothesis class. We prove that our bound is tight in the number of classes. Finally, we compare our bound with the related ones and provide a simplified version of the bound for the multi-class classification with kernel based hypotheses. 
\end{abstract}
\qquad\qquad  {\bf Keywords:} statistical learning, multi-class classification, excess risk 
bound

\section{Introduction}
The principal goal of the statistical learning theory is to provide a framework for studying the problems of a statistical nature and characterize the performance of learning algorithms in order to facilitate the design of better learning algorithm.

The statistical learning theory of supervised binary classification is by now pretty well developed, while its multi-class extension contains numerous statistical challenges. Multi-class classification problems widely arise in everyday practice in various domains, ranging from ranking to computer vision. 
  
For binary classification problems a quite good distribution-free characterization of risk bounds is given via VC dimension. Tighter data-dependent bounds are known in terms of Rademacher complexity or covering numbers. These bounds correctly describe a finite sample performance of learning algorithms.
  
Bounding classification risk for multi-class problems is much less straightforward. Recently, finite sample performance of multi-class learning algorithms was given by means of Natarajan dimension \cite{daniely2014optimal,danielymulticlass}.  An interesting VC-dimension based bound for the risk of large margin mutti-class classifiers is provided in \cite{guermeur2007vc}. 

These estimates give a quite tight data-independent bound on the risk of multi-class classification methods.  On the other hand data-dependent characterization of algorithm quality usually give much better estimates for practical problems.

Rademacher complexity bounds seem to be one of the tightest way to estimate data-dependent finite-sample performance of learning algorithms \cite{koltchinskii2002empirical,bartlett2003rademacher}. There is a lot of progress in risk estimation for binary classification problems~\cite{bartlett2005local,boucheron2013concentration}. 

For multi-class learning problems the situation is more delicate. A seminal paper of Koltchinskii \& Panchenko \cite{koltchinskii2002empirical} provides Rademacher complexity based margin risk bound. The main drawback of this bound is a quadratic dependence on the number of classes, which makes the bound hardly applicable to real-life huge-scale problems of computer vision or text classification. In spite of numerous research there was only a slight improvement of this bound \cite{mohri2012foundations,cortes2013multi}. 

\paragraph{Contribution.} The main contributions of this paper are 
\begin{itemize}
  \item[$a)$] a new Rademacher complexity based bound for large-margin multi-class classifiers. The bound is linear in the number of classes which improves quadratic dependence of formerly the best Rademacher complexity bounds \cite{koltchinskii2002empirical,cortes2013multi};
  \item[$b)$] a new lower bound on the Rademacher complexity of multi-class margin classification methods. This means that sub-linear  in the number of classes Rademacher complexity based bound is hardly possible for multi-class margin classifiers in a standard (unconstrained) model. But it is still possible to provide better bounds in terms of their dependence on the number of classes under other models or extra assumptions \cite{allwein2001reducing,dietterich1995solving,zhang2004statistical}. 
\end{itemize}
\paragraph{Paper structure.} The paper consists of four parts. In the second part of the paper, we present the theoretical contribution, namely new Rademacher complexity bounds. It is followed by a discussion of related works and comparison the proposed bound with other multi-class complexity bounds. 
 
 \section{Multi-class learning guarantees}

We consider a standard multi-class classification framework. Let $\X$ be a set of observations and~$\Y$, $|\Y|< \infty$ be a set of labels respectively. Let $(\X\times \Y, \A, P)$ be a probability space and let $\F$ be a class of measurable functions from $(\X,\A)$ into $\bR$.  Let $\{(x_i, y_i)\}$ be a sequence of i.i.d. random variables taking values in $(\X\times\Y, \A)$ with common distribution $P$. 
We assume that this sequence is defined on a probability space $(\Omega, \Sigma, \bP)$. Let $P_n$ be the empirical measure associated with the sample~$S = \{(x_i, y_i)\}_{i=1}^n$.


We assume that the labels take values in a finite set $\Y$ with $|\Y| = k$. Let $\tilde{\F}$ be a class of functions from $S$ into $\bR$. 
A function $f\in \tilde{\F}$ predicts a label $y \in \Y$ for an example $x\in S$ iff
\begin{gather}
f(x,y) > \max\limits_{y\neq y'} f(x,y')  
\end{gather}
The margin of a labeled example $(x,y)$ is defined as
\begin{gather}\label{eq:margin}
  m_f(x,y) := f(x,y) - \max_{y\neq y'} f(x,y'),
\end{gather}
so $f$ misclassifies the labeled example $(x, y)$ iff $m_f (x, y) \le 0$. 

Let
\begin{gather*}
  \F := \{ f ( \cdot , y ): y \in \Y , f(\cdot, y)\in \tilde{\F}_y\}.
\end{gather*}

In a more common situation all scoring function belongs to same class $\tilde\F$.

We refer to the empirical Rademacher complexity of the class $\F$ as
\[{\widehat\fR}_n(\F) = \bE_{\e}\sup\limits_{f\in \F} \frac{1}{n}\sum\limits_{i=1}^n \e_i f(x_i),\]
where $\e_1, \dots, \e_n$ is independent $\{\pm 1\}$-valued random variables. Then the Rademacher complexity of $\F$ is $\fR_n(\F) = \bE {\widehat \fR}_n(\F)$.



The following theorem states an upper bound for the classification error of $k$-class classifier. This result improves theorem 11 of \cite{koltchinskii2002empirical}, theorem 1 of \cite{cortes2013multi} and theorem 8.1 of \cite{mohri2012foundations} by a factor of $k$.

\begin{theorem}\label{thm:rad-upper}
    For all $t>0$,
    \begin{multline*}
    \bP\biggl\{\exists f\in \tilde \F: P\{m_f \le 0\} > \\
    \inf\limits_{\delta \in (0,1]} \left[P_n\{m_f \le \delta\} + \frac{4k}{\delta} \fR_n(\F)  + \left(\frac{\log\log_2(2/\delta)}{n}\right)^{1/2} + \frac{t}{\sqrt{n}}\right]\biggr\} \le 2 \exp(-2 t^2)
    \end{multline*}
\end{theorem}


Later we show that theorem \ref{thm:rad-upper} 
give a tight bound on the multi-class complexity. 


Let $\M_k(\F_1, \dots, \F_k)$ be a class of functions such that
\begin{gather}\label{eq:margins}
  \M_k(\F_1, \dots, \F_k) = \{\forall m \in \M_k: m(x,y) = f(x,y) - \max\limits_{y\neq y'} f(x, y'), f(x, y) \in \F_y\}.
\end{gather}


Prior to the proof of the theorem one needs to proof the following lemma.
\begin{lemma}\label{lem:redemacher_up}
Let $\M_k(\F_1, \dots, \F_k)$ be a class of margin functions over $\F_1, \dots, \F_k$ defined in \ref{eq:margins}. Then for any i.i.d. sample $S_n = \{(x_i, y_i)\}_{i=1}^n$ of size $n$ holds 
\[{\widehat \fR}_n(\M_k(\F_1,\dots, \F_k)) \le \sum\limits_{j=1}^k {\widehat \fR}_n(\F_j).\]
\end{lemma}
\begin{proof}
We provide a proof of the lemma in the case $\F \doteq \F_1 = \dots = \F_k$. It can be easily extended into a more general case.  For a single class $\F$ the class of margin functions $\M_k(\F)$ has a form
\begin{gather*}
  \M_k(\F) \doteq \{\forall m \in \M_k: m(x,y) = f(x,y) - \max\limits_{y\neq y'} f(x, y')\}.
\end{gather*}
Let $m_f(x,y)^{(\Y'|\Y)}$ be a partial margin of the object $(x, y)$ taken with respect to the subset $\Y'$ of the set of classes, $\Y'\subseteq\Y$:
\[
m_f(x,y)^{(\Y'|\Y)} \doteq \begin{cases}
f(x, y) - \max\limits_{\substack{y'\in \Y'\\ y'\neq y}} f(x, y'), & \text{ if } y\in\Y'\\
- \max\limits_{\substack{y'\in \Y'}} f(x,y'), & \text{ if } y\not\in\Y
\end{cases}
\]

Let $\M^{\Y'}_k(\F) \doteq \{\forall m \in \M^{\Y'}_k(\F): m = m_f^{(k|\Y')}(x_i, y_i), f\in \F\}$.

The proof is by induction on the size of $\Y'$. 
%
%
%
       Note that $\M^{\Y}_k (\F) = \M_k(\F)$ and 
       \[
       \M^{\{1\}}_k(x,y) = \begin{cases}
        f(x, y), & \text{ if } y = 1\\
       -f(x, y), & \text{ if } y \neq 1
       \end{cases}.
       \]
      
      Denote by $\delta(y,y')$ the indicator of $y = y'$
      \[\delta(y, y') = \begin{cases}
      1, &\text{ if } y = y'\\
      0, &\text{ if } y \neq y'
      \end{cases}\]
      
       Then for $\Y' = \{y\}$ holds
       \[
         {\widehat \fR_n}(\M_k^{\Y'}(\F)) = \bE_{\e} \sup\limits_{f\in \F} \frac{1}{n}\sum\limits_{i=1}^n \e_i  (2\delta(y_i,y) - 1) f(x_i)  = \bE_{\e} \sup\limits_{f\in \F} \frac{1}{n}\sum\limits_{i=1}^n \e_i f(x_i)  = {\widehat\fR}_n(\F),
       \]
       because a binary sequence $\delta(y_i, y)$ is independent of the class of functions $\F$ and the Rademacher variables~$\{\e_i\}_{i=1}^n$. Therefore, the induction base is proved. 
       
       The induction hypothesis is that for any $\Y' \subset \Y$, $|\Y'| \le t$ the Rademacher complexity of $\M_k^{\Y'}$ satisfies
       \begin{gather}\label{eq:ind-hyp}
         {\widehat\fR}_n(\M_k^{\Y'}(\F)) \le |\Y'|{\widehat\fR}_n(\F).
       \end{gather}

       If $\Y' = \Y$ the statement is proved, otherwise the set $\Y\setminus \Y'$ is not empty.
       Then for any $\tilde y \in \Y\setminus \Y'$ and i.i.d. sample $S = \{(x_i, y_i)\}_{i=1}^n$ holds
       \begin{multline*}
       {\widehat\fR}_n(\M_k^{\Y' \cup \tilde y}(\F)) = 
             \bE_{\e} \sup\limits_{f\in \F}\frac{1}{n}\biggl\{\sum\limits_{\substack{(x_i, y_i)\in S\\ y_i = \tilde y}} \e_i \{f(x_i,y_i) - \max\limits_{\substack{y\in \Y'\\y_\neq y_i}}f(x_i, y)\} \\ - \sum\limits_{\substack{(x_i, y_i)\in S\\ y_i \neq \tilde y}}\e_i \max\{f(x_i,\tilde y), \max\limits_{y\in \Y'} f(x_i, y)\}\biggr\} 
       \end{multline*}
       Note that $\max\{f_1, f_2\} = \frac{f_1 + f_2}{2} + \frac{|f_1-f_2|}{2}.$
       
       Then 
       \begin{multline*}
       \!\!\!\!\!{\widehat \fR}_n(\M^{\Y' \cup \tilde y}_k(\F)) = 
             \bE_{\e} \sup\limits_{f\in \F}\frac{1}{n} \biggl\{
             \sum\limits_{\substack{(x_i, y_i)\in S\\ y_i = \tilde y}} \e_i(f(x_i,y_i) - \max\limits_{\substack{y\in \Y'\\y_\neq y_i}}f(x_i, y)) - \\ 
             \!\!\!\!\!\!\!\!\!\!\!\!\!\!\!\!\!\!\!\!\!\!\!\!\!\!\!\!\!\!\!\!\!\!\!\!\!\!\!\!\!\!\!\!\!\!\!\!\!\!\!\!\!\! \frac{1}{2}\sum\limits_{\substack{(x_i, y_i)\in S\\ y_i \neq \tilde y}} \e_i\left\{f(x_i,\tilde y) + \max\limits_{y\in \Y'} f(x_i, y)) - \left|f(x_i,\tilde y) - \max\limits_{y\in \Y'} f(x_i, y)\right|\right\} \biggr\}              \le \\
             \!\!\!\!\!\!\!\!\!\!\!\!\!\!\bE_{\e} \sup\limits_{f\in \F}\frac{1}{2n}\sum\limits_{i=1}^n \e_i (2\delta(y_i, \tilde y) - 1)f(x_i, \tilde y) +
             \bE_{\e} \sup\limits_{f\in \F}\frac{1}{2n}\sum\limits_{i=1}^n \e_i(1 - 2\delta(y_i, \tilde y))\max\limits_{y\in \Y'} f(x_i, y) +\\
             \;\bE_{\e} \sup\limits_{f\in \F}\frac{1}{2n}\sum\limits_{i=1}^n \e_i\biggl\{\delta(y_i, \tilde y) (f(x_i,\tilde y) - \max\limits_{y\in \Y'} f(x_i, y)) + (1 - \delta(y_i, \tilde y)) \left|f(x_i,\tilde y) - \max\limits_{y\in \Y'} f(x_i, y)\right|\biggr\} = \\
             \frac{{\widehat \fR}_n(\F)}{2} + \frac{{\widehat \fR}_n(\M^{\Y'}_k(\F))}{2}+ \bE_{\e} \sup\limits_{f\in \F}\frac{1}{2n}\sum\limits_{i=1}^n \e_i \biggl\{\delta(y_i, \tilde y)(f(x_i,\tilde y) - \max\limits_{y\in \Y'} f(x_i, y)) \\
             + (1 - \delta(y_i, \tilde y)) \left|f(x_i,\tilde y) - \max\limits_{y\in \Y'} f(x_i, y)\right|\biggr\}
       \end{multline*}
       
       Note, that $x + y\rightarrow x+|y|$ is a 1-Lipschitz. Thus by Talagrand's contraction inequality (see theorem~4.12, {p.~112--114} of~\cite{ledoux1991probability} and more appropriate lemma 4.2, p. 78--79 of \cite{mohri2012foundations}) holds
       \begin{multline*}
         {\widehat \fR}_n(\M^{\Y' \cup \tilde y|\Y}) \le \frac{{\widehat \fR}_n(\F)}{2} + \frac{{\widehat \fR}_n(\M^{\Y'})}{2} + \\ 
         \!\!\!\!\!\!\!\!\!\!\!\!\!\! \bE_{\e} \sup\limits_{f\in \F}\frac{1}{2n}\sum\limits_{i=1}^n  \e_i(2\delta(y_i, \tilde y) - 1)(f(x_i,m) - \max\limits_{y\in \Y'} f(x_i, y)) \le \\ 
         \frac{{\widehat \fR}_n(\F)}{2} + \frac{{\widehat \fR}_n(\M^{\Y'})}{2}  + \bE_{\e} \sup\limits_{f\in \F}\frac{1}{2n}\sum\limits_{i=1}^n  \e_i(2\delta(y_i, \tilde y) - 1)f(x_i,m) + \\
         \bE_{\e} \sup\limits_{f\in \F}\frac{1}{2n}\sum\limits_{i=1}^n  \e_i(1 - 2\delta(y_i, \tilde y))\max\limits_{y\in \Y'} f(x_i, y) = \\ 
         {\widehat \fR}_n(\F) + {\widehat \fR}_n(\M^{\Y'}_k(\F)) \le (|\Y'|+1){\widehat \fR}_n(\F),
       \end{multline*}
       where the last but one inequality holds by the inductive hypothesis, ineq. \ref{eq:ind-hyp}).
       This completes the inductive proof.
%
\end{proof}

\begin{proof}[Proof of the theorem \ref{thm:rad-upper}]
Following to \cite{koltchinskii2002empirical} consider 2 sequences $\{\delta_j\}_{j\ge 1}$ and $\{\e_j\}_{j\ge 1}$, $\e_j \in (0,1)$.

The standard Rademacher complexity margin bound (theorem 4.4, p. 81--82 of \cite{mohri2012foundations}) gives for any fixed $\delta_t$ and $\e_t$:
\[\bP\left\{P(m_f(x,y) < 0) - P_n (m_f(x,y)< \delta_t) \ge \frac{2}{\delta_j}\fR(\M_t(\F)) + \varepsilon_j \right\} \le \exp(-2n\e^2_j).\]

Then by choosing $\e_j  = \frac{t}{\sqrt{n}} + \sqrt{\frac{\log j}{n}}$ and applying the union bound 
\begin{multline*}
  \bP\left\{\exists\;j: P(m_f(x,y) < 0) - P_n (m_f(x,y)< \delta_j) \ge \frac{2}{\delta_j}\fR(\M_k(\F)) + \varepsilon_j\right\} \\ \le \sum\limits_{j\ge 1} \exp(-2n\e_j^2) \le 
  \exp(-2t^2) \sum\limits_{j\ge 1} \exp(-2\log j) = \frac{\pi^2}{6}\exp(-2n t^2) < 2 \exp(-2n t^2).
\end{multline*}
We choose $\delta_k = 1/2^k$, then ${2}/{\delta_j} \le 4/\delta$.  By lemma \ref{lem:redemacher_up} we have $\fR(\M_k(\F)) \le k\fR(\F)$ which proofs the theorem.
\end{proof}


Below we present a Rademacher complexity bounds for multi-class kernel learning in a simplified form.
Let $\fK:\X\times\X \to \bR$ be a positive definite symmetric kernel and $\Phi: \X\to \mathbb{H}$ be a 
feature mapping associated to $\mathfrak{K}$. In the multi-class setting a 
 family of kernel-based hypotheses $\mathcal{H}_{k,p}$ is defined for any $p\ge 1$ as
\[\H_{\fK,p} = \{(x,y)\in \X\times \Y\to w_y\cdot \Phi(x): W = (w_1,\dots, w_k)^\mathrm{T}, \|W\|_{\bH, p} \le \Lambda\},\]
where $\|W\|_{\bH}^p = (\sum_{i=1}^k \|w_i\|_{\bH}^p)^{1/p}$. The labels are assigned according to 
$\arg\max\limits_{y\in \Y} \langle w_y, \Phi(x) \rangle$.

The following bound is a corollary of the theorem \ref{thm:rad-upper}.
\begin{theorem}
  Let $\fK: \X\times\X\to \bR$ be a positive definite symmetric kernel and let $\Phi:\X\to \bH$ be the associated feature mapping function. Assume that there exists $R > 0$ such that $\fK(x,x)\le R^2$ for all $x\in \X$. Then, for any $t > 0$ the following multi-class classification generalization bounds hold for all hypotheses $h\in \bH_{K,p}$
    \begin{gather*}
    \bP\left\{\exists f\in \tilde \F: P\{m_f \le 0\} > 
P_n\{m_f \le \delta\} + \frac{2k}{\delta} \sqrt{\frac{R^2\Lambda^2}{n}}  + \frac{t}{\sqrt{n}}\right\} \le \exp(-2 t^2)
    \end{gather*}
\end{theorem}


Below we proof that the bound on the Rademacher complexity of the class $\fR_n(\M_k(\F_1, \dots, \F_k)$ is tight. 
Let $\F^{j}_t = \{f:\bR \to [-1;+1]\}$ be a class of functions such that 
\[\F^{j}_t \ni f(x) = \begin{cases}
-1, &\text{ if } x\not\in [j; j+1]\\
+1\text{ or }-1, &\text{ if } x\in [j; j+1]
\end{cases}\]
and moreover each $f\in \F^{j}_t$ has in $(j, j+1)$ no more than $t$ discontinuity points. We refer to $\F_0$ as the class of functions takes $-1$ over real line.

Denote
\[\F^*_t = \left\{\max\{f_1,f_2,\dots, f_k\}, f_i \in \F^{j}_t\right\} \text{ and } \F_t = \bigcup_{j=1}^m \F^{j}_t.\]

Note, that all the classes $\{\F^j_t\}_{j=1}^k$,  $\F_t$ and $\{\F^*_t\}$ for a fixed $t$ satisfy the conditions of the central limit theorem.


Let $\fR^*_n(\F^j_t)$ be a Rademacher complexity of $\F^j_t$ defined with respect to the interval $(j, j+1)$ only
\[\fR^*_n(\F^j_t) = \sup\limits_{f\in\F^j_t} \frac{1}{n} \sum\limits_{i=1}^n \e_i f(x_i)\Indicator{x_i \in (j, j+1)}.\]
\begin{lemma}\label{lem:sup01}
  Let $P^\X$ be a uniform distribution over the domain $\mathcal{X} = [1; k+1]$. Then for any $C > 0$ there exists $t = t(C,k)$ such that for any sample $S_n = \{x_i\}_{i=1}^n$ of size $n$ drawn i.i.d. from $P^\X$ and any $j$, $1\le j \le k$ holds
  \[
  \fR^*_n(\F^j_t) \ge C\, \fR_n(\F_0)
  \]
since $n \ge n_0$, $n_0 = n_0(t)$.
\end{lemma}
\begin{proof}
By theorem 5.3.3. of \cite{talagrand2014upper} for any sequences $t_1, \dots, t_m$ in $\ell^2$ such that 
\[\ell\neq\ell' \Rightarrow \|t_{\ell} - t_{\ell'}\| \ge a\]
and 
\[\forall \ell \le m \Rightarrow \|t_{\ell}\|_{\infty} \le b\]
the following lower bound for Rademacher process holds
\begin{gather}\label{eq:533}
  \bE_{\e}\sup\limits_{f\in \F} \sum\limits_{i=1}^n f(x_i)\e_i\ge \frac{1}{L}\min\left\{a\sqrt{\log m}, \frac{a^2}{b}\right\},
\end{gather}
for some absolute constant $L$.

By the standard chaining argument the Rademacher complexity of the class $\F_0$ satisfies
\[\fR_n(\F_0) \le \frac{C_0}{\sqrt{n}},\]
for some absolute constant $C_0 = C_0(\F) > 0$ independent of $n$.

Let objects $x^1, \dots, x^{n_j}$ belong to $(j, j+1)$ are ordered in such a way that $(x^i - x^j)(i-j) \ge 0$ for all $i,j$. 
Note that for any such sequence there exist functions $\{f_1, \dots, f_{2^{\lfloor n_j/t \rfloor}}\} \in \F_{t+1}$ such that the function $f_j$ assigns $+1$ to objects $\{x^{st +1}, \dots, x^{st+t}\}$, $s: 1\le s\le \lfloor n_j/t \rfloor$ iff a binary representation of $j$ contains 1 in $s$-th digit from the right. Otherwise it assigns to $-1$ to $\{x^{st +1}, \dots, x^{st+t}\}$.

Then by the equation \ref{eq:533} the following lower bound on Rademacher complexity of the class ${\widehat \F}_j$,
${\widehat \F}_j~=~\{f_1, \dots, f_{2^{\lfloor n_j/t \rfloor}}\}$ takes place  
\[\bE_{\e}\sup\limits_{f\in \F} \frac{1}{n} \sum\limits_{i=1}^n \e_i f(x_i) \Indicator{x\in (j, j+1)} \ge \frac{1}{L}\min\left\{\frac{n_j}{n}\sqrt{2 - \frac{2t}{n_j}},  \frac{2t\sqrt{n_j}}{n}\right\},\quad f\in \widehat \F_j\]
for some absolute constant $L$ stated by the inequality \ref{eq:533}. 

Remind that the median for Binomial distribution with parameter $1/k$ is one of the integers $\{\lfloor n/k\rfloor-1, \lfloor n/k\rfloor, \lfloor n/k\rfloor+1\}$. Then the number of objects in $(j,j+1)$ is $n/k-2$ or more with probability at least~1/2. 

Therefore, if $n \ge 16kt^2$, $t\ge 1$
\begin{multline*}
  \fR_n({\widehat \F}_j) = \bE\,\bE_{\e}\sup\limits_{f\in \F} \frac{1}{n} \sum\limits_{i=1}^n f(x_i)\e_i\ge \frac{1}{L}\min\left\{\frac{1}{2k}\sqrt{2 - \frac{4tk}{n}}, \frac{2t}{\sqrt{nk}}\right\} \ge \\ \min\left\{\frac{1}{2kL}, \frac{2t}{L\sqrt{nk}}\right\} = \frac{2t}{L\sqrt{nk}}.
\end{multline*}
Then it is sufficient to choose $t\ge C_0 C L \sqrt{k}/2$ and $n \ge 16kt^2$ as above to satisfy the conditions of the lemma.
\end{proof}

%
%
%
%
%
%
%

\begin{theorem}\label{th:rad-lower}
  Let $P^\X$ be a uniform distribution over the domain $\mathcal{X} = [1; k+1]$ and $P^\Y$ concentrated on a single class $k+1$.  Then for any sample $S_n = \{(x_i, y_i)\}_{i=1}^n$ of size $n$ drawn i.i.d. from $P^\X\times P^\Y$ and any $\e > 0$  for the Rademacher complexity of the margin class $\M_{k+1} = (\F_t^1, \dots, \F_t^{k+1}, \F_0)$ holds  
  \[
  \fR_n(\M_{k+1}) \ge (1 - \e)\sum\limits_{j=1}^k\fR_n(\F_t^j)  
  \]
  for some large enough $t = t(\e, k)$ independent of $n$ and all $n \ge n_0$, $n_0 = n_0(t)$.
\end{theorem}
\begin{proof}
By the symmetry under negation of classes $\F^j_t$ in $(j, j+1)$ and definition of the class $\F_t^*$ we have
\begin{gather*}
  \fR_n(\F_t^*) = \sum\limits_{j=1}^n \fR^*(\F_t^*) \ge \left(1- \frac{1}{C}\right) \sum\limits_{j=1}^k\fR_n(\F_t^j) = k\left(1- \frac{1}{C}\right)\fR_n(\F_t^j), \quad j': 1\le j \le k
\end{gather*}
where $C= 1/\e$ is defined in accordance with the lemma \ref{lem:sup01}.

Note that the Rademacher complexity of $\M_{k+1}(\F_t^1, \dots, \F_t^k, \F_0)$ is at least the same as the Rademacher complexity of $\F_t^*$ by the construction of $\M_{k+1}$ and $\F_t^1, \dots, \F_t^k$. This proofs the lemma. 
\end{proof}

A similar bound holds for the Rademacher complexity of the classes $\F_t$ and $\M_{k+1}(\F_t)$ respectively. Note, that this bound is effectively the lower bound to the estimate of the theorem \ref{thm:rad-upper} in the sense that the bound there can not improved based on the Rademacher complexity estimates only if one put no assumptions on the behavior of the function class (e.g. small covering number bound or small~VC dimension).

\section{Related works and discussion}
A number of works are devoted to bounding the risk of multi-class classification methods. One popular approach to solving a problem with multiple classes is to reduce it to a sequence of binary classification problems.  In terms of risk dependence on the number of classes a great breakthrough was done with the design of error--correcting output codes (ECOC) for multi-class classification~\cite{dietterich1995solving,allwein2001reducing,beygelzimer2009error}. 

In spite of some very promising results concerning ECOC  Rifkin \& Klautau argued in~\cite{rifkin2004defense} that the classical approaches, such as one-vs-all classification, is at least as preferable as error-correcting codes from the practical point of view. 

Another approach is to define a score function on the point-label pairs and choose a label with the highest score (one-vs-all classification method can be considered from this point of view as well). It is natural to characterize the risk bounds of these methods in terms of classification margin $\delta$ equals to the gap between the highest score and the second highest score (see def. \ref{eq:margin} for details). 

\paragraph{Multi-class SVM extension.} Among the methods that share scoring-based paradigm, one should mention the Weston \& Watkins multi-class extension of SVM \cite{weston1998multi}. An improved version multi-class SVM as well as the improved margin risk bound of the order $\tilde O(k^2/n\delta^2)$ were presented by Crammer \& Singer in \cite{crammer2002learnability,crammer2002algorithmic}. 

\paragraph{Rademacher complexity bounds.} 
Currently Rademacher complexity as well as combinatorial dimension estimates seem to be among of the most powerful tools to get strong enough risk bounds for multi-class classification. The important property of Rademacher complexity based bounds is that the bounds are applicable in arbitrary Banach spaces and do not depend on the dimension of the feature space directly.  

Koltchinksii \& Panchenko introduced a margin-based bound for multi-class classification in terms of Rademacher complexities \cite{koltchinskii2002empirical,koltchinskii2001some}. The bound was slightly improved (by a constant factor prior to the Rademacher complexity term) in a series of subsequent works \cite{mohri2012foundations,cortes2013multi}. 

The main drawback of these state-of-the-art bounds for multi-class classification is a quadratic dependence on the number of classes which makes the bounds unreliable for practical problems with a considerable number of classes. 

The principal contribution of this paper is a new Rademacher complexity based upper bound with a linear complexity w.r.t. the number of classes. Moreover  we provide the lower bound on Rademacher complexity of margin-based multi-class algorithms. Up to a constant factor it matches to the upper bound. Than means that the bound can not be improved without further assumptions.

\paragraph{Covering number based bounds.} Zhang in \cite{zhang2004statistical,zhang2002covering} studied covering number bounds for the risk of the multi-class margin classification. Based on the $\ell^{\infty}$ covering number bound estimate for the Rademacher complexity of kernel learning problem he obtained asymptotically better rates in the number of classes $k$ (see tab.~1) than those proposed in our paper. 

Note, that Zhang's analysis is based on some extra assumptions (not really too restrictive) about underlying hypothesis class and the loss function used. We suppose that the results of \cite{zhang2004statistical} are appreciated from the theoretical point of view but still quite limited for practice. This is due to high overestimate (from a practical perspective) of the Rademacher complexity of the hypothesis class by a $\ell^{\infty}$ covering number based bound. It should also be noted that Zhang's bound are valid only for learning kernel-based hypothesis and have some extra poly-logarithmic dependence on the number of labeled examples.

Related results for metric spaces with low doubling dimension were obtained by Kontorovich \cite{kontorovich2014maximum}, who used nearest neighbors method to improve the dependence on the number of classes in favor of (doubling) dimension dependence. We should note as well that his approach allows to speed-up  multi-class learning algorithms.

We gather margin based bounds applicable for learning functions in Hilbert space the tab.~1.

\begin{table}[ht]
  \centering
  \begin{tabular}{c|l}
    Upper bound, $\tilde O(\cdot)$ & Paper \\
    \hline    \hline
    $\frac{k^2}{\delta \sqrt{n}}$ & Koltchinskii \& Panchenko, \cite{koltchinskii2002empirical} \\ 
    & Cortes et al., \cite{cortes2013multi}, \\
    & Mohri et al. \cite{mohri2012foundations} \\ 
    $\frac{k}{\delta^2\sqrt{n}}$ & Guermeur, \cite{guermeur2010ensemble} \\
    $\frac{1}{\delta}\sqrt{\frac{k}{n}}$ & Zhang, \cite{zhang2004statistical}  \\
    $\frac{k^2}{\delta^2n}$ & Crammer \& Singer, \cite{crammer2002learnability}  \\
    $\frac{k}{\delta \sqrt{n}}$ & this paper  
  \end{tabular}
  \caption{Dimension-free margin-based bounds for multi-class classification.}
\end{table}

\paragraph{Combinatorial dimension bounds.} Natarajan dimension  was introduced in \cite{natarajan1989learning} in order to characterize multi-class PAC  learnability. It exactly matches the notion of Vapnik-Chervonenkis dimension in the case of two classes. A number of results concerning risk bounds in terms of Natarajan dimension were proved in \cite{danielymulticlass,daniely2014optimal,bendavid1995characterizations,daniely2012multiclass}.  A closely related but more powerful notion of graph dimension was introduced in \cite{danielymulticlass,daniely2014optimal}. VC-dimension based bounds for multi-class learning problems were obtained in \cite{allwein2001reducing}.

Natarajan and graph dimensions are very useful tools for obtaining multi-class classification risk bounds. The main drawback of these bounds is that they are data-independent. In this sense, we believe that the bounds proposed in this paper are much stronger than the Natarajan/graph dimension bounds same as that of Rademacher complexity bounds are stronger than the VC dimension bounds for binary classification. 

We also note that VC dimension bounds as well as Natarajan dimension bounds are usually dimension dependent \cite{daniely2014optimal}, which makes them hardly applicable for practical huge scale problems (such as typical computer vision problems). 


Guermeur in \cite{guermeur2007vc,guermeur2010ensemble} gave a bound for scale-sensitive analog of Natarajan dimension $\tilde d_{Nat}$.  In Hilbert space for a class of linear functions it can be bounded in terms of the margin as $\tilde O(k^2/\delta^2)$ which leads to the risk decay rate of the order $\tilde O(k/\delta^2\sqrt{n})$ (see tab.~1).

We gather the bounds above in the tab.~2. Note, that the bound of the order $\tilde O(d_{Nat}/n)$ is valid in a separable case only. 
\begin{table}[ht]
\centering
\begin{tabular}{c|l}
    Upper bound,     $\tilde O(\cdot)$ & Paper \\
    \hline\hline
    $\frac{\log k}{\delta}\sqrt{\frac{d_{VC}}{n}}$ & Allwein et al., \cite{allwein2001reducing}\\
    $\frac{\log k}{\delta}\sqrt{\frac{\tilde d_{Nat}}{n}}$ & Guermeur, \cite{guermeur2010ensemble} \\
    $\frac{d_{Nat}}{n}$ & Daniely et al., \cite{daniely2014optimal}
  \end{tabular}
  \caption{Combinatorial dimension based upper bounds for multi-class classification.}

\end{table}
A clear comparison between various multi-class classification methods is provided in \cite{daniely2012multiclass}. Lower bounds on Natarajan dimension and sample complexity of multi-class classification methods provided in \cite{danielymulticlass,daniely2014optimal}. It was shown in \cite{danielymulticlass,daniely2014optimal} that for multi-class linear classifiers the bounds on Natarajan dimension can be as poor as $\Omega(dk)$, where $d$ is a feature space dimension and $k$ are a number of classes. In this work we provide a linear (in the number of classes) lower bound on the Rademacher complexity of the multi-class margin class of functions (see th.~\ref{th:rad-lower} for details). 

A preliminary version of the upper bounds (theorem \ref{thm:rad-upper}) with slightly poor dependence on $k$ was presented by the first author in context of semi-supervised multi-class classification on the workshop ``Frontiers of High Dimensional Statistics, Optimization, and Econometrics'' in February 2015. The risk bounds stated in this paper were presented in the final form on March 25-th at the main seminar of Institute for Information Transmission Problems (IITP RAS). In July 2015 the authors were notified be their colleagues that similar results were proposed independently by  Kuznetsov et al. and presented on ICML Workshop on Extreme Classification.\footnote{Vitaly Kuznetsov, Mehryar Mohri and Umar Syed. Rademacher complexity margin bounds for learning with a large number of classes. In \emph{In ICML 2015 Workshop on Extreme Classification. Lille, France, July 2015}} and in \cite{kuznetsov2014multi}. Still we suppose that the bounds presented in this paper are much stronger than the ones presented by Kuznetsov et al. in the sense that we prove explicit lower bounds as well. This shows that the bound which we proved in theorem \ref{thm:rad-upper} is tight, i.e. linear dependence on the number of classes is inevitable if no further assumptions are made. 

\section{Conclusion.}
In this paper we propose new state-of-the-art Rademacher complexity based upper bounds for the risk of multi-class margin classifiers. The bound depends linearly in in the number of classes. We prove as well that the bound can not be further improved based on the Rademacher complexities only. Still it is possible to provide a better estimates for the excess risk of multi-class classification using other techniques or supplementary assumptions.

\section{Acnowledgement.}
We are grateful to Massih-Reza Amini and Zaid Harchaoui for the problem setting and useful suggestions. We would also like to thank Anatoli Juditsky, Grigorii Kabatianski, Vladimir Koltchinskii, Axel Munk, Arkadi Nemirovski and Vladimir Spokoiny for helpful discussions. 

The research of the first author is supported by the Russian Foundation of Basic Research, grants 14-07-31241 mol\_a and 15-07-09121 a. The second author is supported by the Russian Science Foundation, grant 14-50-00150.

\bibliographystyle{chicago}
\bibliography{arxiv.bib}

\end{document}